\documentclass{article}

\newif\ifsup\suptrue

     \usepackage[final]{neurips_2020}

\usepackage[utf8]{inputenc} 
\usepackage[T1]{fontenc}    
\usepackage{comment}
\usepackage{url}            
\usepackage{booktabs}       
\usepackage{amsfonts}       
\usepackage{nicefrac}       
\usepackage{microtype}      
\usepackage[svgnames]{xcolor}
\usepackage{smile}
\usepackage{enumerate,natbib,mathrsfs}
\usepackage{algorithm}
\usepackage{algorithmic}
\usepackage{tablefootnote}
\usepackage{extarrows}
\usepackage[OT1]{fontenc}
\usepackage[bookmarks=false]{hyperref}

\hypersetup{
  pdftex,
  pdffitwindow=true,
  pdfstartview={FitH},
  pdfnewwindow=true,
  colorlinks,
  linktocpage=true,
  linkcolor=Red,
  urlcolor=Red,
  citecolor=Blue
}
\usepackage{stmaryrd}
\usepackage{graphicx}

\newcommand{\supp}{\text{supp}}

\newcommand{\KL}{\text{KL}}

\usepackage[textsize=tiny]{todonotes}
\title{High-Dimensional Sparse Linear Bandits}

\author{%
  Botao Hao\\
  Deepmind \\
  \texttt{haobotao000@gmail.com} \\
   \And
   Tor Lattimore \\
   Deepmind \\
   \texttt{lattimore@google.com} \\
   \AND
   Mengdi Wang \\
   Department of Electrical Engineering\\
  Princeton University \\
   \texttt{mengdiw@princeton.edu} \\
}

\begin{document}

\maketitle

\begin{abstract}
Stochastic linear bandits with high-dimensional sparse features are a practical model for a variety of domains,
including personalized medicine and online advertising \citep{bastani2020online}. 
We derive a novel $\Omega(n^{2/3})$ dimension-free minimax regret lower bound for sparse linear bandits in the data-poor  
regime where the horizon is smaller than the ambient dimension and where the feature vectors admit a well-conditioned exploration distribution.
This is complemented by a nearly matching upper bound for an explore-then-commit algorithm showing that that $\Theta(n^{2/3})$ is the optimal rate
in the data-poor regime.
The results complement existing bounds for the data-rich regime and provide another example where carefully balancing 
the trade-off between information and regret is necessary.   
Finally, we prove a dimension-free $\cO(\sqrt{n})$ regret upper bound under an additional assumption on the magnitude of the signal for relevant features.
\end{abstract}
\section{Introduction}

Stochastic linear bandits generalize the standard reward model for multi-armed bandits by associating each action with a feature vector and assuming the mean reward is
the inner product between the feature vector and an unknown parameter vector
\citep{auer2002using, dani2008stochastic, rusmevichientong2010linearly, chu2011contextual, abbasi2011improved}. 

In most practical applications, there are many candidate features but no clear indication about which are relevant. Therefore, it is crucial to consider stochastic linear bandits 
in the high-dimensional regime but with low-dimensional structure, captured here by the notion of sparsity. 
Previous work on sparse linear bandits has mostly focused on the data-rich regime, where the time horizon is larger than the ambient dimension
\citep{abbasi2012online, carpentier2012bandit, wang2018minimax, kim2019doubly, bastani2020online}. The reason for studying the data-rich regime is partly justified by 
minimax lower bounds showing that for smaller time horizons the regret is linear in the worst case.

Minimax bounds, however, do not tell the whole story. A crude maximisation over all environments hides much of the rich structure of linear bandits with sparsity.
We study sparse linear bandits in the high-dimensional regime when the ambient dimension is much larger than the time horizon. In order to sidestep existing lower bounds,
we refine the minimax notion by introducing a dependence in our bounds on the minimum eigenvalue of a suitable exploration distribution over the actions. Similar quantities appear already in the
vast literature on high-dimensional statistics \citep{buhlmann2011statistics,wainwright2019high}.

\paragraph{Contributions}
Our first result is a lower bound showing that $\Omega(n^{2/3})$ regret is generally unavoidable when the dimension is large, even if the action set 
admits an exploration policy for which the minimum eigenvalue of the associated data matrix is large.
The lower bound is complemented by an explore-the-sparsity-then-commit algorithm that first solves a convex optimization problem to find the most informative design in
the exploration stage. The algorithm then explores for a number of rounds by sampling from the design distribution and uses Lasso  \citep{tibshirani1996}
to estimate the unknown parameters. Finally, it
greedily chooses the action that maximizes the reward given the estimated parameters. We derive an $\cO(n^{2/3})$ dimension-free regret
that depends instead on the minimum eigenvalue of the covariance matrix associated with the exploration distribution. Our last result is a 
post-model selection linear bandits algorithm that invokes phase-elimination algorithm \citep{lattimore2019learning} to the model selected 
by the first-step regularized estimator. Under a sufficient condition
on the minimum signal of the feature covariates, we prove that a dimension-free $\cO(\sqrt{n})$ regret is achievable, even if the data is scarce.

The analysis reveals a rich structure that has much in common with partial monitoring, where $\Theta(n^{2/3})$ regret occurs naturally in settings for which
some actions are costly but highly informative \citep{BFPRS14}. A similar phenomenon appears here when the dimension is large relative to the horizon.
There is an interesting transition as the horizon grows, since $\cO(\sqrt{dn})$ regret is optimal in the data rich regime.

\begin{table}\label{table:comparsion}
\centering
\caption{Comparisons with existing results on regret upper bounds and lower bounds for sparse linear bandits. Here, $s$ is the sparsity, $d$ is the feature dimension, $n$ is the number of rounds, $K$ is the number of arms, $C_{\min}$ is the minimum eigenvalue of the data matrix for an exploration distribution (\ref{def:minimum_eigen}) and $\tau$ is a problem-dependent parameter that may have a complicated form and vary across different literature.}

\scalebox{0.94}{
\begin{tabular}{ |l|c|c|c| } 
 \hline
 \textbf{Upper Bound}& Regret& Assumptions & Regime\\ 
 \hline
 \cite{abbasi2012online} & $\cO(\sqrt{sdn})$ & none &  rich \\ 
 \hline
 \cite{sivakumar2020structured} & $\cO(\sqrt{sdn})$ & adver. + Gaussian noise&  rich \\ 
  \hline
 \cite{bastani2020online} & $\cO(\tau K s^2(\log(n))^2)$ & compatibility condition&  rich \\
 \hline
  \cite{wang2018minimax}& $\cO(\tau K s^3\log (n))$ & compatibility condition &  rich\\
 \hline
  \cite{kim2019doubly} & $\cO(\tau s\sqrt{n})$ & compatibility condition &  rich\\ 
 \hline
 \cite{lattimore2015linear}
  &$\cO(s\sqrt{n})$& action set is hypercube &  rich\\
  \hline
   This paper (Thm. \ref{thm:upper_bound})
  &$\cO(C_{\min}^{-2/3}s^{2/3}n^{2/3})$& action set spans $\mathbb R^d$ &  poor\\
  \hline
   This paper (Thm. \ref{thm:improved_upper_bound})
  &$\cO(C_{\min}^{-1/2}\sqrt{sn })$& action set spans $\mathbb R^d$ + mini. signal & rich\\
  \hline
 \textbf{Lower Bound}&& &   \\ 
 \hline
 Multi-task bandits\tablefootnote{Section 24.3 of \cite{lattimore2018bandit}} & $\Omega(\sqrt{sdn})$ & N.A.&  rich  \\ 
 \hline
  This paper (Thm. \ref{thm:minimax_lower_bound}) & $\Omega(C_{\min}^{-1/3}s^{1/3}n^{2/3})$ & N.A. &  poor  \\ 
 \hline
\end{tabular}}
\vspace{0.1in}

\end{table}

\paragraph{Related work}
Most previous work is focused on the data-rich regime. For an arbitrary action set,  \cite{abbasi2012online} proposed an 
online-to-confidence-set conversion approach that achieves a $\cO(\sqrt{sdn})$ regret upper bound, where $s$ is a known upper bound on the sparsity. 
The algorithm is generally not computationally efficient, which is believed to be unavoidable.
Additionally, a $\Omega(\sqrt{sdn})$ regret lower bound for data-rich regime was established in Section 24.3 of \cite{lattimore2018bandit}, which means polynomial dependence on $d$ is generally not avoidable without additional assumptions. 

For this reason, it recently
became popular to study the contextual setting, where the action set changes from round to round and to
careful assumptions are made on the context distribution. The assumptions are chosen so that techniques from
high-dimensional statistics can be borrowed. Suppose $\tau$ is a problem-dependent parameter that may have a complicated form and varies across different literature. \cite{kim2019doubly} developed a doubly-robust Lasso bandit approach with an $\cO(\tau s\sqrt{n})$ upper bound but required the average of the feature vectors 
for each arm satisfies the compatibility condition \citep{buhlmann2011statistics}. \cite{bastani2020online} and \cite{wang2018minimax} considered 
a multi-parameter setting (each arm has its own underlying parameter) and assumed the distribution of contexts satisfies a variant of 
the compatibility condition as well as other separation conditions.  \cite{bastani2020online} derived a $\cO(\tau K s^2(\log(n))^2)$ upper bound 
and was sharpen to $\cO(\tau K s^2\log(n))$ by \cite{wang2018minimax}, where $K$ is the number of arms. Although those results are 
dimension-free, they require strong assumptions on the context distribution that are hard to verify in practice. 
As a result, the aforementioned regret bounds involved complicated problem-dependent parameters that may be very large when the assumptions fail to hold.

Another thread of the literature is to consider specific action sets. \cite{lattimore2015linear}
proposed a selective explore-then-commit algorithm that only works when the action set is exactly the binary hypercube. They derived an optimal $\cO(s\sqrt{n})$ upper bound
as well as an optimal gap-dependent bound.
\cite{sivakumar2020structured} assumed the action set is generated adversarially but perturbed artificially by some standard Gaussian noise. 
They proposed a structured greedy algorithm to achieve an $\cO(s\sqrt{n})$ upper bound. 
\cite{deshpande2012linear} study the data-poor regime in a Bayesian setting but did not consider sparsity. 
\cite{carpentier2012bandit} considered a special case where the action set is the unit sphere and the noise 
is vector-valued so that the noise becomes smaller as the dimension grows. We summarize the comparisons in Table \ref{table:comparsion}.

\section{Problem setting}
In the beginning, the agent receives a compact action set $\cA\subseteq \mathbb R^d$, 
where $d$ may be larger than the number of rounds $n$. At each round $t$, the agent chooses an action $A_t\in\cA$ and receives a reward
\begin{equation}\label{def:sparse_linear}
    Y_t = \langle A_t, \theta\rangle + \eta_t\,,
\end{equation}
where $(\eta_t)_{t=1}^n$ is a sequence of independent standard
Gaussian random variables and $\theta\in\mathbb R^d$ is an unknown parameter vector. 
We make the mild boundedness assumption that for all $x \in \cA$, $\|x\|_\infty \leq 1$.
The parameter vector $\theta$ is assumed to be $s$-sparse: 
\begin{equation*}
    \|\theta\|_0 = \sum_{j=1}^d \ind\{\theta_j\neq 0\} \leq s.
\end{equation*}
The Gaussian assumption can be relaxed to conditional sub-Gaussian assumption for the regret upper bound, but is necessary for the regret lower
bound. 
The performance metric is the cumulative expected
regret, which measures the difference between the
expected cumulative reward collected by the omniscient policy that knows $\theta$ and that of the learner. 
The optimal action is $x^* = \argmax_{x\in\cA}\langle x, \theta\rangle$ and the regret of the agent when facing the bandit determined by $\theta$ is
\begin{equation*}
   R_{\theta}(n) = \mathbb E\left[\sum_{t=1}^n \langle x^*, \theta\rangle - \sum_{t=1}^n Y_t\right] \,,
\end{equation*}
where the expectation is over the interaction sequence induced by the agent and environment.
Our primary focus is on finite-time bounds in the data-poor regime where $d \geq n$. 

\paragraph{Notation}
Let $[n] = \{1,2, \ldots, n\}$. For a  vector $x$ and positive semidefinite matrix $A$, we let $\|x\|_A=\sqrt{x^{\top}Ax}$ be 
the weighted $\ell_2$-norm and $\sigma_{\min}(A),\sigma_{\max}(A)$ be the minimum eigenvalue and maximum eigenvalue of $A$, respectively. 
The cardinality of a set $\cA$ is denoted by $|\cA|$. The support of a vector $x$, $\supp(x)$, is the set of indices $i$ such that $x_i\neq 0$. And $\ind\{\cdot\}$ is an indicator function.
The suboptimality gap of action $x \in \cA$ is $\Delta_x = \langle x^*, \theta \rangle - \langle x, \theta \rangle$ and the
minimum gap is $\Delta_{\min} = \min\{\Delta_x : x \in \cA, \Delta_x > 0\}$.

\section{Minimax lower bound} 

As promised, we start by proving a kind of minimax regret lower.
We first define a quantity that measures the degree to which there exist good exploration distributions over the actions.

\begin{definition}
\label{def:minimum_eigen}
Let $\cP(\cA)$ be the space of probability measures over $\cA$ with the Borel $\sigma$-algebra and define 
\begin{equation*}
    C_{\min}(\cA) = \max_{\mu\in\cP(\cA)} \sigma_{\min}\Big(\mathbb E_{A\sim \mu}\big[AA^{\top}\big]\Big) \,.
\end{equation*}
\end{definition}

\begin{remark}
$C_{\min}(\cA) > 0$ if and only if $\cA$ spans $\mathbb R^d$. Two illustrative examples are the hypercube and probability simplex.
Sampling uniformly from the corners of each set shows that $C_{\min}(\cA) \geq 1$ for the former and $C_{\min}(\cA) \geq 1/d$ for the latter.
\end{remark}

The next theorem is a kind of minimax lower bound for sparse linear bandits. 
The key steps of the proof follow, with details and technical lemmas deferred to
\ifsup
Appendix \ref{sec:main_proof}.
\else
the supplementary material.
\fi

\begin{theorem}\label{thm:minimax_lower_bound}
Consider the sparse linear bandits described in Eq.~\eqref{def:sparse_linear}. 
Then for any policy $\pi$ there exists an action set $\cA$ with $C_{\min}(\cA)>0$ and $s$-sparse parameter $\theta\in\mathbb R^d$ such that
\begin{equation}\label{eqn:minimax_lower_bound}
    R_{\theta}(n)\geq \frac{\exp(-4)}{4}\min\Big(C_{\min}^{-\tfrac{1}{3}}(\cA)s^{\tfrac{1}{3}}n^{\tfrac{2}{3}},\sqrt{dn}\Big).
\end{equation}
\end{theorem}
Theorem \ref{thm:minimax_lower_bound} holds for any data regime and suggests an intriguing transition between $n^{2/3}$ and $n^{1/2}$ regret, depending
on the relation between the horizon and the dimension. When $d>n^{1/3}s^{2/3}$ the bound is
$\Omega(n^{2/3})$, which is independent of the dimension. On the other hand, when $d\leq n^{1/3}s^{2/3}$, we recover the standard $\Omega(\sqrt{sdn})$ 
dimension-dependent lower bound up to a $\sqrt{s}$-factor. 
In Section~\ref{sec:upper_bound}, we prove that the $\Omega(n^{2/3})$ minimax lower bound is tight by presenting a nearly matching upper bound in the data-poor regime. 

\begin{remark}
Theorem \ref{thm:minimax_lower_bound} has a worst-case flavor. For each algorithm we construct 
a problem instance with the given dimension, sparsity and value of $C_{\min}$ for which the stated regret bound holds. 
The main property of this type of hard instance is that it should include a informative but high-regret action set 
such that the learning algorithm should balance the trade-off between information and regret. This leaves the possibility 
for others to create minimax lower bound for their own problem.  
\end{remark}

\begin{proof}[Proof of Theorem~\ref{thm:minimax_lower_bound}]
The proof uses the standard information-theoretic machinery, but with a novel construction and KL divergence calculation.

\textbf{Step 1: construct a hard instance.} We first construct a low regret action set $\cS$ and an informative action set $\cH$ as follows:
\begin{equation}\label{eqn:action_set}
    \begin{split}
       & \cS = \Big\{x\in\mathbb R^d\Big|x_j\in\{-1, 0, 1\} \ \text{for} \ j\in[d-1], \|x\|_1=s-1, x_d = 0\Big\} \,,\\
       & \cH = \Big\{x\in\mathbb R^d\Big|x_j\in\{-\kappa, \kappa\} \ \text{for} \ j\in[d-1], x_d = 1\Big\} \,,
    \end{split}
\end{equation}
where $0<\kappa\leq1$ is a constant. The action set is the union $\cA = \cS \cup \cH$ and let
\begin{equation*}
    \theta = \big(\underbrace{\varepsilon, \ldots, \varepsilon}_{s-1}, 0, \ldots, 0, -1\big) \,,
\end{equation*}
where $\varepsilon>0$ is a small constant to be tuned later.
Because $\theta_d = -1$, actions in $\cH$ are associated with a large regret.
On the other hand, actions in $\cH$ are also highly informative, which hints towards an interesting tradeoff between regret and information. 
Note that $\cH$ is nearly the whole binary hypercube, while actions in $\cS$ are $(s-1)$-sparse.
The optimal action is in the action set $\cA$:
\begin{equation}\label{def:optimal_action}
    x^* = \argmax_{x\in\cA}\langle x,\theta\rangle = \big(\underbrace{1,\cdots, 1}_{s-1}, 0,\ldots, 0\big)\in \cA \,.
\end{equation}

\textbf{Step 2: construct an alternative bandit}. The second step is to construct an alternative bandit $\tilde{\theta}$ that is hard to distinguish from $\theta$ and
for which the optimal action for $\theta$ is suboptimal for $\tilde \theta$ and vice versa. Denote $\mathbb P_{\theta}$ and $\mathbb P_{\smash{\tilde \theta}}$ as the measures on the sequence
of outcomes $(A_1, Y_1, \ldots, A_n, Y_n)$ induced by the interaction between a fixed bandit algorithm 
and the bandits determined by $\theta$ and $\tilde{\theta}$ respectively. 
Let $\mathbb E_{\theta}, \mathbb E_{\tilde{\theta}}$ be the corresponding expectation operators.
We denote a set $\cS'$ as 
\begin{equation}\label{def:S_prime}
\begin{split}
     \cS' = \Big\{x\in\mathbb R^d\Big|&x_j\in\{-1, 0, 1\} \ \text{for} \ j\in\{s,s+1,\ldots, d-1\} \,, \\
     & x_j = 0 \ \text{for} \ j=\{1,\ldots, s-1, d\}, \|x\|_1=s-1\Big\} \,.
\end{split}
\end{equation}
Clearly, $\cS'$ is a subset of $\cS$ and for any $x\in\cS'$, its support has no overlap with $\{1,\ldots, s-1\}$. Then we denote
\begin{equation}\label{def:x_tilde}
   \tilde{x} = \argmin_{x\in\cS'}\mathbb E_{\theta} \left[\sum_{t=1}^n  \langle A_t, x \rangle^2\right] \,,
\end{equation}
and construct the alternative bandit $\tilde{\theta}$ as  
\begin{equation}\label{def:theta_alt}
    \tilde{\theta} = \theta + 2\varepsilon \tilde{x} \,.
\end{equation}
Note that $\tilde{\theta}$ is $(2s-1)$-sparse since $\tilde{x}$ belongs to $\cS'$ that is a $(s-1)$-sparse set. This design guarantees the optimal arm $x^*$ in bandit $\theta$ is suboptimal in alternative bandit $\tilde{\theta}$ and the suboptimality gap for $x^*$ in bandit $\tilde{\theta}$ is $\max_{x\in\cA}\langle x-x^*, \tilde{\theta}\rangle = (s-1)\varepsilon. $
Define an event 
    \begin{equation*}
        \cD = \left\{\sum_{t=1}^n \ind(A_t\in\cS)\sum_{j=1}^{s-1}A_{tj}\leq \frac{n(s-1)}{2}\right\} \,.
\end{equation*}
The next claim shows that when $\cD$ occurs, the regret is large in bandit $\theta$, while if it does not occur, then the regret is large in bandit $\smash{\tilde \theta}$.
The detailed proof is deferred to \ifsup
Appendix \ref{sec:claim_regret_lower}.
\else
the supplementary material.
\fi
\begin{claim}\label{claim:regret_lower} 
Regret lower bounds with respect to event $\cD$:
\begin{equation*}
    \begin{split}
      R_{\theta}(n)\geq\frac{n(s-1)\varepsilon}{2}\mathbb P_{\theta}(\cD) \qquad \text{and} \qquad   
      R_{\tilde{\theta}}(n) \geq \frac{n(s-1)\varepsilon}{2}\mathbb P_{\tilde{\theta}}(\cD^c) \,.
    \end{split}
\end{equation*}
\end{claim}

By the Bretagnolle--Huber inequality (Lemma \ref{lem:kl} in the appendix),
\begin{equation*}
    \begin{split}
         R_{\theta}(n) +  R_{\tilde{\theta}}(n) 
         \geq \frac{n(s-1)\varepsilon}{2} \Big(\mathbb P_{\theta}(\cD) + \mathbb P_{\tilde{\theta}}(\cD^c)\Big)\geq \frac{n(s-1)\varepsilon}{4}  \exp\Big(-\KL\big(\mathbb P_{\theta}, \mathbb P_{\tilde{\theta}}\big)\Big) \,,
    \end{split}
\end{equation*}
where $\KL(\mathbb P_{\theta}, \mathbb P_{\smash{\tilde{\theta}}})$ is the KL divergence between probability measures $\mathbb P_{\theta}$ and $\mathbb P_{\smash{\tilde{\theta}}}$.

\textbf{Step 3: calculating the KL divergence.} We make use of the following bound on the KL divergence 
between $\mathbb P_{\theta}$ and $\mathbb P_{\smash {\tilde \theta}}$, which
formalises the intuitive notion of information. When the KL divergence is small, the algorithm is unable
to distinguish the two environments. The detailed proof is deferred to 
\ifsup
Appendix \ref{sec:claim_KL}. 
\else
the supplementary material.
\fi

\begin{claim}\label{claim:KL_bound}
Define $T_n(\cH) = \sum_{t=1}^n \ind(A_t\in\cH)$.
The KL divergence between $\mathbb P_{\theta}$ and $\mathbb P_{\smash{\tilde{\theta}}}$ is upper bounded by
\begin{equation}\label{eqn:KL_bound}
   \KL\left(\mathbb P_{\theta}, \mathbb P_{\smash{\tilde{\theta}}}\right) \leq 2\varepsilon^2\left(\frac{n(s-1)^2}{d} + \kappa^2(s - 1)\mathbb E_{\theta}[T_n(\cH)]\right) \,.
\end{equation}

\end{claim}
The first term in the right-hand side of the bound is the contribution from actions in the low-regret action set $\cS$, while the second term
is due to actions in $\cH$. The fact that actions in $\cS$ are not very informative is captured by the presence of the dimension in the denominator of the first term.
When $d$ is very large, the algorithm simply does not gain much information by playing actions in $\cS$.
When $T_n(\cH)<1/(\kappa^2(s-1)\varepsilon^2)$, it is easy to see 
\begin{equation}\label{eqn:case1}
      R_{\theta}(n) +  R_{\tilde{\theta}}(n) \geq \frac{n(s-1)\varepsilon}{4} \exp\left(-\frac{2n\varepsilon^2(s-1)^2}{d} \right)\exp(-2) \,.
\end{equation}
On the other hand, when $T_n(\cH)>1/(\kappa^2\varepsilon^2(s-1))$, we have 
\begin{equation}\label{eqn:case2}
    \begin{split}
        R_{\theta}(n) \geq \mathbb E_{\theta}[T_n(\cH)]\min_{x\in\cH} \Delta_{x} \geq \frac{1}{\kappa^2\varepsilon^2(s-1)} + \frac{1-\kappa}{\kappa^2\varepsilon} \,,
    \end{split}
\end{equation}
since $\min_{x\in\cH} \Delta_{x}= 1 + (s-1)\varepsilon (1-\kappa)$ from the definition of $\cH$ and $\theta$. 

\textbf{Step 4: conclusion.} Combining the above two cases together, we have
\begin{equation}
\begin{split}
    R_{\theta}(n) +  R_{\tilde{\theta}}(n)\geq \min\left(\left(\frac{ns\varepsilon}{4} \right)  \exp\left(-\frac{2\varepsilon^2s^2 n}{d} \right)\exp(-2),\, \frac{1}{\kappa^2\varepsilon^2s} + \frac{1-\kappa}{\kappa^2\varepsilon}\right) \,,
\end{split}
\end{equation}
where we replaced $s-1$ by $s$ in the final result for notational simplicity. 
Consider a sampling distribution $\mu$ that uniformly samples actions from $\cH$. A simple calculation shows that $C_{\min}(\cA)\geq C_{\min}(\cH) \geq \kappa^2>0$. This is due to
\begin{equation*}
    \sigma_{\min}\left(\sum_{x\in\cH}\mu(x) xx^{\top}\right) = \sigma_{\min}\Big(\mathbb E_{X\sim\mu}[XX^{\top}]\Big) = \kappa^2 \,,
\end{equation*}
where each coordinate of the random vector $X\in \mathbb R^d$ is sampled independently uniformly from $\{-1, 1\}$.
In the data poor regime when $d\geq n^{1/3}s^{2/3}$, we choose $\varepsilon = \kappa^{-2/3} s^{-2/3} n^{-1/3}$ such that
\begin{equation*}
\begin{split}
     \max(R_{\theta}(n), R_{\tilde{\theta}}(n))&\geq R_{\theta}(n) +  R_{\tilde{\theta}}(n)\\
     &\geq \frac{\exp(-4)}{4}\kappa^{-\tfrac{2}{3}}s^{\tfrac{1}{3}}n^{\tfrac{2}{3}}\geq \frac{\exp(-4)}{4}C_{\min}^{-\tfrac{1}{3}}(\cA)s^{\tfrac{1}{3}}n^{\tfrac{2}{3}} \,.
\end{split}
\end{equation*}
Finally, in the data rich regime when
$d < n^{1/3}s^{2/3}$ we choose $\varepsilon = \sqrt{d/(ns^2)}$ such that the exponential term is a constant, and then 
\begin{equation*}
    \max(R_{\theta}(n), R_{\tilde{\theta}}(n))\geq R_{\theta}(n) +  R_{\tilde{\theta}}(n)  \geq \frac{\exp(-4)}{4}\sqrt{dn} \,. \qedhere
\end{equation*}
\end{proof}

\section{Matching upper bound}\label{sec:upper_bound}
We now propose a simple algorithm based on the explore-then-commit paradigm\footnote{ Explore-then-commit template is also considered in other works \citep{deshmukh2018simple} but both the exploration 
and exploitation stages are very different. \cite{deshmukh2018simple} considers simple regret minimization while we focus on cumulative regret minimization. } and show that the minimax lower bound in Eq.~\eqref{eqn:minimax_lower_bound} is more or less achievable.
As one might guess,
the algorithm has two stages. First it solves the following optimization problem to find the most informative design: 
\begin{equation}\label{def:opti}
    \begin{split}
       \max_{\mu\in\cP(\cA)}  \ \sigma_{\min}\Big(\int_{x\in\cA}xx^{\top} d \mu(x)\Big)\,.
    \end{split}
\end{equation}
In the exploration stage, the agent samples its
actions from $\hat{\mu}$ for $n_1$ rounds, collecting a data-set $\{(A_1, Y_1), \ldots, (A_{n_1}, Y_{n_1})\}$. The agent uses the data collecting in the exploration stage 
to compute the Lasso estimator $\hat{\theta}_{n_1}$. In  
the commit stage, the agent executes the greedy action for the rest $n-n_1$ rounds. The detailed algorithm is summarized in Algorithm \ref{alg:estc}. 
\begin{remark}
The minimum eigenvalue is concave \citep{boyd2004convex}, which means that the solution to \eqref{def:opti} can be approximated efficiently using standard tools such as
CVXPY \citep{diamond2016cvxpy}.
\end{remark}

{\small
\begin{algorithm}[htb!]
	\caption{Explore the sparsity then commit (ESTC)}
	\begin{algorithmic}[1]\label{alg:estc}
		\STATE
		\textbf{Input:} time horizon $n$, action set $\cA$, exploration length $n_1$, regularization parameter $\lambda_1$;
		
		\STATE Solve the optimization problem in Eq.~\eqref{def:opti} and denote the solution as $\hat{\mu}$.
		\FOR{$t= 1, \cdots, n_1$}
\STATE Independently pull arm $A_t$ according to $\hat{\mu}$ and receive a reward: $Y_t = \langle A_t, \theta\rangle + \eta_t.$
\ENDFOR
\STATE Calculate the Lasso estimator \citep{tibshirani1996}:
\begin{equation}\label{eqn:Lasso_est}
    \hat{\theta}_{n_1} = \argmin_{\theta\in\mathbb R^d}\Big(\frac{1}{n_1}\sum_{t=1}^{n_1}\big(Y_t-\langle A_t, \theta\rangle\big)^2 + \lambda_1\|\theta\|_1\Big).
\end{equation}

	\FOR{$t=n_1+1$ to $n$}
	\STATE Take greedy actions $A_t = \argmin_{x\in\cA}\langle \hat{\theta}_{n_1}, x\rangle.$
	\ENDFOR
	\end{algorithmic}
\end{algorithm}
}

The following theorem states a regret upper bound for Algorithm \ref{alg:estc}. The proof is deferred to 
\ifsup
Appendix \ref{sec:proof_upper}.
\else
the supplementary material.
\fi

\begin{theorem}\label{thm:upper_bound}
Consider the sparse linear bandits described in Eq.~\eqref{def:sparse_linear} and assume the action set $\cA$ spans $\mathbb R^d$. Suppose $R_{\max}$ is an upper bound of maximum expected reward such that $ \max_{x\in\cA}\langle x,\theta \rangle\leq R_{\max}$. In Algorithm \ref{alg:estc}, we choose
\begin{equation}\label{eqn:exploration_length}
    n_1 = n^{2/3}(s^2\log (2d))^{1/3}R_{\max}^{-2/3} (2/C_{\min}^2(\cA))^{1/3},
\end{equation}
and $\lambda_1 = 4\sqrt{\log (d) / n_1}$. Then the following regret upper bound holds,
\begin{equation}\label{eqn:upper_bound_estc}
    R_{\theta}(n) \leq (2\log (2d)R_{\max})^{\tfrac{1}{3}}C_{\min}^{-\tfrac{2}{3}}(\cA)s^{\tfrac{2}{3}}n^{\tfrac{2}{3}}+ 3nR_{\max}\exp(-c_1n_1).
\end{equation}
\end{theorem}
Together with the minimax lower bound in Theorem \ref{thm:minimax_lower_bound}, we can argue that ESTC algorithm is minimax optimal in time horizon $n$ in the data-poor regime.
\begin{remark}
The regret upper bound Eq.~\eqref{eqn:upper_bound_estc} may still depend on $d$ because $1/C_{\min}(\cA)$ could be as large as $d$. 
Indeed, if the action set is the standard basis vectors, then the problem reduces to the standard multi-armed bandit for which the minimax regret is $\Theta(\sqrt{dn})$,
even with sparsity.
If we restrict our attention to the class of action set such that $1/C_{\min}(\cA)$ is dimension-free, then we have a dimension-free upper bound.
\end{remark}

\begin{remark}
Another notion frequently appearing in high-dimensional statistics is the restricted eigenvalue condition. Demanding a lower bound on the restricted eigenvalue is weaker than 
the minimum eigenvalue, which can lead to stronger results. As it happens, however, the two coincide in the lower bound construction. The upper bound may also be sharpened, 
but the resulting optimization problem would (a) depend on the sparsity $s$ and (b) the objective would have a complicated structure for which an efficient algorithm is not yet apparent.
\end{remark}

\begin{remark}
There is still a $(s/C_{\min}(\cA))^{1/3}$ gap between the lower bound (Eq.~\eqref{eqn:minimax_lower_bound}) and upper bound (Eq.~\eqref{eqn:upper_bound_estc}) ignoring logarithmic factor. 
We conjecture that the use of $\ell_1 / \ell_{\infty}$ inequality when proving Theorem~\ref{thm:upper_bound} is quite conservative. 
Specifically, we bound the following using the $\ell_1$-norm bound of Lasso (see Eq.~\eqref{eqn:regret_decom} in the Appendix~\ref{sec:proof_upper} for details),
\begin{equation*}
     \big\langle \theta-\hat{\theta}_{n_1}, x^*-A_t\big\rangle \leq \big\|\theta-\hat{\theta}_{n_1}\big\|_1\big\|x^*-A_t\big\|_{\infty}\lesssim \sqrt{\frac{s^2\log (d)}{n_1}}.
\end{equation*}
The first inequality ignores the sign information of $\hat{\theta}_{n_1}$ and the correlation between $x^*-A_t$ and $\hat{\theta}_{n_1}$. 
A similar phenomenon has been observed by \cite{javanmard2018debiasing} and resolved by means of a delicate leave-one-out analysis to decouple the correlation. 
An interesting question is whether or not a similar technique could be used in our case to improve the above bound to $\sqrt{s \log(d) / (n_1)}$, 
closing the gap between regret upper bound and lower bound. On the other hand, surprisingly, even in the classical statistical settings there are still gaps between upper and lower bounds in terms of $C_{\min}(\cA)$ \citep{6034739}. We speculate that the upper bound may be improvable, though at present we do not know how to do it.
\end{remark}

\begin{remark}
The algorithm uses knowledge of the sparsity to tune the length of exploration in Eq.~\eqref{eqn:exploration_length}. 
When the sparsity is not known, the length of exploration can be set to $n_1= n^{2/3}$.
The price is an additional factor of $\cO(s^{1/3})$ to regret. This is an advantage relative to the algorithm by \cite{abbasi2012online}, for which knowledge of the
sparsity is apparently essential for constructing the confidence set.
\end{remark}

\begin{remark}
We do not expect explicit optimism-based algorithms \citep{dani2008stochastic,rusmevichientong2010linearly, chu2011contextual, abbasi2011improved} or implicit ones, 
such as Thompson sampling \citep{agrawal2013thompson}, to achieve the minimax lower bound in the data-poor regime.
The reason is that the optimism principle does not balance the trade-off between information 
and regret, a phenomenon that has been observed before in linear and structured bandits \citep{lattimore2017end, combes2017minimal, hao2019adaptive}.
\end{remark}

\section{Improved upper bound}
In this section, we show that under additional minimum signal condition, the restricted phase elimination algorithm can achieve a sharper $\cO(\sqrt{sn})$ regret upper bound. 

The algorithm shares similar idea with \cite{carpentier2012bandit} that includes feature selection step and restricted linear bandits step. In the feature selection step, the agent pulls a certain number of rounds $n_2$ following $\hat{\mu}$ as in \eqref{def:opti}. Then Lasso is used to conduct the feature selection. Based on the support Lasso selects, the algorithm invokes phased elimination algorithm for linear bandits \citep{lattimore2019learning} on the selected support.

{\small
\begin{algorithm}[htb!]
	\caption{Restricted phase elimination}
	\begin{algorithmic}[1]\label{alg:support}
		\STATE
		\textbf{Input:} time horizon $n$, action set $\cA$, exploration length $n_2$, regularization parameter $\lambda_2$;
		\STATE Solve the optimization problem Eq.~\eqref{def:opti} and denote the solution as $\hat{\mu}$.
		\FOR{$t= 1, \cdots, n_2$}
\STATE Independently pull arm $A_t$ according to $\hat{\mu}$ and receive a reward: $Y_t = \langle A_t, \theta\rangle + \eta_t.$
\ENDFOR
\STATE Calculate the Lasso estimator $\hat{\theta}_{n_2}$ as in Eq.~\eqref{eqn:Lasso_est} with $\lambda_2$.
\STATE Identify the support: $\hat{S} = \supp(\hat{\theta}_{n_2})$.
	\FOR{$t=n_2+1$ to $n$}
	\STATE Invoke phased elimination algorithm for linear bandits on $\hat{S}$.
	\ENDFOR
	\end{algorithmic}
\end{algorithm}
}

\begin{cond}[Minimum signal]\label{con:min_signal}
We assume there exists some known lower bound $m>0$ such that $
\min_{j\in \supp(\theta)} |\theta_j| >m.$
\end{cond}

\begin{theorem}\label{thm:improved_upper_bound}
Consider the sparse linear bandits described in Eq.~\eqref{def:sparse_linear}. We assume the action set $\cA$ spans $\mathbb R^d$ as well as $|\cA| = K<\infty$ and suppose Condition \ref{con:min_signal} holds. Let $n_2=C_1 s\log(d)/(m^2C_{\min}(\cA))$ for a suitable large constant $C_1$ and choose $\lambda_2= 4\sqrt{\log (d)/n_2}$. Denote $\phi_{\max} = \sigma_{\max}(\sum_{t=1}^{n_2}A_tA_t^{\top}/n_2)$. Then the following regret upper bound of Algorithm \ref{alg:support} holds,
\begin{equation}\label{eqn:improved_upper}
    R_{\theta}(n) \leq C\Big(\frac{s\log (d)}{m^2C_{\min}(\cA)} + \sqrt{\frac{9\phi_{\max}\log (Kn)}{C_{\min}(\cA)}}\sqrt{sn}\Big),
\end{equation}
for universal constant $C>0$.
\end{theorem}
When $C_{\min}(\cA)$ is dimension-free and $m\geq (s\log^2(d)/C_{\min}(\cA)n)^{1/4}$, we reach an $\cO(\sqrt{sn})$ regret upper bound.
The proof is deferred to 
\ifsup
Appendix~\ref{sec:proof_improved_bound}. 
\else
the supplementary material.
\fi
It utilizes the sparsity and variable screening property of Lasso. More precisely, under minimum signal condition, the Lasso estimator can identify 
all the important covariates, i.e., $\supp(\hat{\theta}_{n_1})\supseteq \supp(\theta)$. And the model Lasso selected is sufficiently sparse, i.e. $|\supp(\hat{\theta}_{n_1})|\lesssim s$. Therefore, it is enough to query linear bandits algorithm on $\supp(\hat{\theta}_{n_1})$. 

\begin{remark}
It is possible to remove the dependency of $\phi_{\max}$ in the Eq.~\eqref{eqn:improved_upper} using more dedicated analysis, using theorem 3 in \cite{belloni2013least}. 
The reason we choose a phase elimination type algorithm is that it has the optimal regret guarantee when the size of action set is moderately large. 
When the action set has an infinite number of actions, we could switch to the linear UCB algorithm \citep{abbasi2011improved} or appeal to a discretisation argument.
\end{remark}

\section{Experiment}
 We compare ESTC (our algorithm) with LinUCB \citep{abbasi2011improved} and doubly-robust (DR) lasso bandits \citep{kim2019doubly}. For ESTC, we use the theoretically suggested length of exploration stage. For LinUCB, we use the theoretically suggested confidence interval. For DR-lasso, we use the code made available by the authors on-line.
\begin{itemize}
    \item \textbf{Case 1: linear contextual bandits.} We use the setting in Section 5 of \cite{kim2019doubly} with $N=20$ arms, dimension $d=100$, sparsity $s=5$. At round $t$, we generate the 
action set from $N(0_N, V)$, where $V_{ii} = 1$ and $V_{ik} = \rho^2$ for every $i\neq k$. Larger $\rho$ corresponds to high correlation setting that is more favorable to DR-lasso. 
The noise is from $N(0, 1)$ and  $\|\theta\|_0 = s$.
\item \textbf{Case 2: hard problem instance.} Consider the hard problem instance in the proof of minimax lower bound (Theorem \ref{thm:minimax_lower_bound}), 
including an informative action set and an uninformative action set. Since the size of action set constructed in the hard problem instance grows exponentially with $d$, we uniformly randomly sample  500 actions from the full informative action set and 200 from uninformative action set. 
\end{itemize}

{\textbf{Conclusion:}} The experiments confirm our theoretical findings. Although our theory focuses on the fixed action set setting, ESTC works well in the contextual setting. 
DR-lasso bandits heavily rely on context distribution assumption and almost fail for the hard instance. 
LinUCB suffers in the data-poor regime since it ignores the sparsity information.
 \begin{figure}[h]
 \centering
\includegraphics[width=0.35\linewidth]{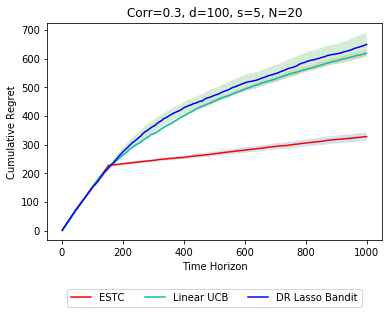}
\includegraphics[width=0.35\linewidth]{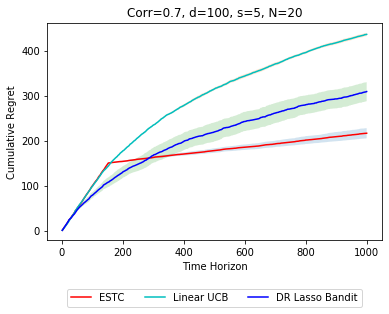}
\includegraphics[width=0.35\linewidth]{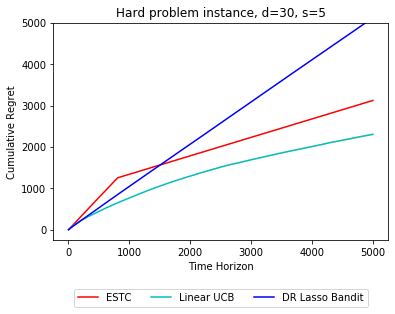}
\includegraphics[width=0.35\linewidth]{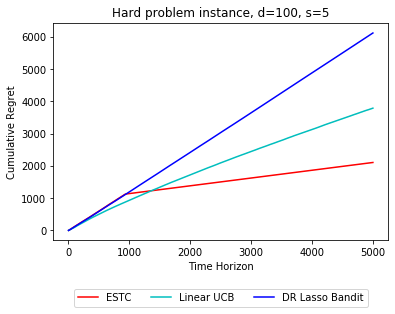}
\label{fig:contextuak}
 \caption{The top two figures are for Case 1 and the bottom two figures are for Case 2.}
\end{figure}

\section{Discussion}

In this paper, we provide a thorough investigation of high-dimensional sparse linear bandits, and show that $\Theta(n^{2/3})$ is the optimal rate
in the data-poor regime. Our work leaves many open problems on how the shape of action set affects the regret that reveals the subtle trade-off between information and regret. For instance, it is unclear how the regret lower bound depends on $C_{\min}(\cA)$ in the data-rich regime and if $C_{\min}(\cA)$ is the best quantity to describe the shape of action set $\cA$.

In another hand, the ESTC algorithm can only achieve optimal regret bound in data poor regime and becomes suboptimal in the data rich regime. It is interesting to have an algorithm to achieve optimal regrets in ``best of two worlds". Information-direct sampling \citep{russo2014learning} might be a good candidate since it delicately balances the trade-off between information and regret which is necessary in the sparse linear bandits.


\paragraph{Broader Impact}
We believe that presented research should be categorized as basic research and we
are not targeting any specific application area. Theorems may inspire new algorithms and theoretical
investigation. The algorithms presented here can be used for many different applications and a
particular use may have both positive or negative impacts. We are not aware of any immediate short
term negative implications of this research and we believe that a broader impact statement is not
required for this paper.

\begin{ack}
Mengdi Wang gratefully acknowledges funding from the U.S. National Science Foundation (NSF) grant CMMI1653435, Air Force Office of Scientific Research (AFOSR) grant FA9550-19-1-020, and C3.ai DTI.
\end{ack}

\bibliographystyle{plainnat}
{\small
\bibliography{ref}
}


\newpage
\clearpage

\appendix
In Appendix \ref{sec:review_hd}, we review some statistical results for sparse linear regression. In Appendix \ref{sec:main_proof}, we provide the proof of main theorems as well as main claims. In Appendix \ref{sec:supporting_lemmas}, we include some supporting lemma for the sake of completeness.
\section{Sparse linear regression}\label{sec:review_hd}

We review some classical results in sparse linear regression. Consider the following sparse linear regression model:
\begin{equation}
    y_i = \langle x_i, \theta^*\rangle+\epsilon_i, i =1,\ldots, n,
\end{equation}
where $\theta^*\in\mathbb R^d$ and $\|\theta^*\|_0=s\leq d$ and the noise $\{\epsilon_i\}_{i=1}^n$ independently follows a zero-mean, $\sigma$-sub-Gaussian distribution. Let the design matrix be $X = (x_1, \ldots, x_n)^{\top}\in\mathbb R^{n\times d}$. Define the Lasso estimator as follows:
\begin{equation*}
    \hat{\theta}_n = \argmin_{\theta}\Big(\frac{1}{n}\sum_{i=1}^n(y_i-\langle x_i, \theta\rangle)^2 + \lambda\|\theta\|_1\Big).
\end{equation*}

\begin{condition}[(Restricted eigenvalues)]\label{con:RE}
Define the cone:
\begin{equation*}
    \mathbb C(S):= \{\Delta\in \mathbb R^d|\|\Delta_{S^c}\|_1\leq 3\|\Delta_S\|_1\},
\end{equation*}
where $S$ is the support set of $\theta^*$. Then there exists some positive constant $\kappa$ such that the design matrix $X\in\mathbb R^{n\times d}$ satisfied the condition 
\begin{equation*}
    \frac{\|X\theta\|_2^2}{n}\geq \kappa \|\theta\|_2^2,
\end{equation*}
for all $\theta\in \mathbb C(S)$.
\end{condition}
\begin{condition}[(Column normalized)]\label{con:column_normal}
Using $X_j\in\mathbb R^n$ to denote the $j$-th column of $X$, we say that $X$ is column-normalized if for all $j=1,2,\ldots, d$,
\begin{equation*}
    \frac{\|X_j\|_2}{\sqrt{n}}\leq 1.
\end{equation*}
\end{condition}
\begin{theorem}\label{thm:Lasso}
Consider an $s$-sparse linear regression and assume design matrix $X\in\mathbb R^{n\times d}$ satisfies the RE condition (Condition \ref{con:RE}) and the column normalization
condition (Condition \eqref{con:column_normal}). Given the Lasso estimator with regularization parameter $\lambda_n=4\sigma\sqrt{\log (d)/n}$, then with probability at least $1-\delta$, 
\begin{itemize}
    \item the estimation error under $\ell_1$-norm (Theorem 7.13 in \cite{wainwright2019high}) of any optimal solution $\hat{\theta}_n$ satisfies 
    \begin{equation*}
    \big\|\hat{\theta}_n -\theta^*\big\|_1\leq \frac{\sigma s}{\kappa}\sqrt{\frac{2\log (2d/\delta)}{n}};
\end{equation*}
    \item the mean square prediction error (Theorem 7.20 in \cite{wainwright2019high}) of any optimal solution $\hat{\theta}_n$ satisfies
\begin{equation*}
    \frac{1}{n}\sum_{i=1}^n\big(x_i^{\top}(\hat{\theta}_n - \theta)\big)^2\leq \frac{9}{\kappa}\frac{s\log (d/\delta)}{n}.
\end{equation*}
\end{itemize}

\end{theorem}

\section{Proofs of main theorems and claims}\label{sec:main_proof}

\subsection{Proof of Claim \ref{claim:regret_lower}}\label{sec:claim_regret_lower}

We first prove the first part. By standard calculations, we have 
\begin{equation*}
    \begin{split}
        R_{\theta}(n) &=\mathbb E_{\theta}\Big[\sum_{t=1}^n\langle x^*, \theta\rangle\Big]-\mathbb E_{\theta}\Big[\sum_{t=1}^n\langle A_t, \theta\rangle \Big]\\
        &= \mathbb E_{\theta}\Big[n(s-1)\varepsilon-\sum_{t=1}^n\ind(A_t\in\cH)\langle A_t, \theta\rangle- \sum_{t=1}^n\ind(A_t\in\cS)\langle A_t, \theta\rangle\Big],
    \end{split}
\end{equation*}
where the last equation is from the definition of $x^*$ in Eq.~\eqref{def:optimal_action}. From the definition of $\cH$ in Eq.~\eqref{eqn:action_set}, the following holds for small enough $\varepsilon$, 
\begin{equation}\label{eqn:T_n}
    \sum_{t=1}^n\ind(A_t\in\cH)\langle A_t, \theta\rangle\leq T_n(\cH)(\kappa(s-1)\varepsilon-1) \leq 0,
\end{equation}
where $T_n(\cH) = \sum_{t=1}^n \ind(A_t\in\cH)$.
Since $\langle A_t, \theta\rangle = \sum_{j=1}^sA_{tj}\varepsilon$ for $A_t\in\cS$, then it holds that
\begin{equation}\label{eqn:bound_R}
    \begin{split}
      R_{\theta}(n)&\geq \mathbb E_{\theta}\Big[  n(s-1)\varepsilon -\sum_{t=1}^n\ind(A_t\in\cS)\sum_{j=1}^{s-1}A_{tj}\varepsilon\Big]\\
      &\geq \mathbb E_{\theta}\Big[ \Big(n(s-1)\varepsilon -\sum_{t=1}^n\ind(A_t\in\cS)\sum_{j=1}^{s-1}A_{tj}\varepsilon\Big)\ind(\cD)\Big]\\
      &\geq \Big( n(s-1)\varepsilon -\frac{n(s-1)\varepsilon}{2}\Big)\mathbb P_{\theta}(\cD)\\
      &=\frac{n(s-1)\varepsilon}{2}\mathbb P_{\theta}(\cD).
    \end{split}
\end{equation}

Second, we derive a regret lower bound of alternative bandit $\tilde{\theta}$.  Denote $\tilde{x}^*$ as the optimal arm of  bandit $\tilde{\theta}$.
By a similar decomposition in Eq.~\eqref{eqn:bound_R},
\begin{equation}\label{eqn:decom5}
    \begin{split}
        R_{\tilde{\theta}}(n) &= \mathbb E_{\tilde{\theta}}\Big[\sum_{t=1}^n\langle \tilde{x}^*, \tilde{\theta}\rangle\Big]-\mathbb E_{\tilde{\theta}}\Big[\sum_{t=1}^n\langle A_t, \tilde{\theta}\rangle \Big]\\
        &= \mathbb E_{\tilde{\theta}}\Big[2n(s-1)\varepsilon - \sum_{t=1}^n\ind(A_t\in\cH)\langle A_t, \tilde{\theta}\rangle - \sum_{t=1}^n\ind(A_t\in\cS)\langle A_t, \tilde{\theta}\rangle\Big]\\
        & \geq \mathbb E_{\tilde{\theta}}\Big[2n(s-1)\varepsilon - \sum_{t=1}^n\ind(A_t\in\cS)\langle A_t, \tilde{\theta}\rangle\Big].
    \end{split}
\end{equation}
where the inequality comes similarly in Eq.~\eqref{eqn:T_n} to show $\sum_{t=1}^n\ind(A_t\in\cH)\langle A_t, \tilde{\theta}\rangle\leq 0$. Next, we will find an upper bound for $\sum_{t=1}^n\ind(A_t\in\cS)\langle A_t, \tilde{\theta}\rangle$. From the definition of $\tilde{\theta}$ in Eq.~\eqref{def:theta_alt}, 
\begin{equation}\label{eqn:decom3}
\begin{split}
     \sum_{t=1}^n\ind(A_t\in\cS)\langle A_t, \tilde{\theta}\rangle &= \sum_{t=1}^n\ind(A_t\in\cS)\langle A_t, \theta + 2\varepsilon\tilde{x}\rangle\\
     &= \sum_{t=1}^n\ind(A_t\in\cS)\langle A_t, \theta\rangle + 2\varepsilon \sum_{t=1}^n\ind(A_t\in\cS)\langle A_t, \tilde{x}\rangle\\
     &\leq\sum_{t=1}^n\ind(A_t\in\cS)\langle A_t, \theta\rangle+2\varepsilon\sum_{t=1}^n\ind(A_t\in\cS)\sum_{j\in\supp(\tilde{x})}|A_{tj}|,
\end{split}
\end{equation}
where the last inequality is from the definition of $\tilde{x}$ in Eq.~\eqref{def:x_tilde}.
To bound the first term, we have
\begin{equation}\label{eqn:decom2}
    \begin{split}
        \sum_{t=1}^n \ind(A_t\in\cS)\langle A_t, \theta \rangle & = \sum_{t=1}^n \ind(A_t\in\cS)\sum_{j=1}^{s-1}A_{tj}\varepsilon\\
        & \leq \varepsilon \sum_{t=1}^n \ind(A_t\in\cS)\sum_{j=1}^{s-1}|A_{tj}|.
    \end{split}
\end{equation}
If all the actions $A_t$ come from $\cS$ which is a $(s-1)$-sparse set, we have 
    \begin{equation*}
    \sum_{t=1}^n \sum_{j=1}^d |A_{tj}|= (s-1)n,
\end{equation*}
which implies 
\begin{equation}\label{eqn:decom4}
\begin{split}
    &\sum_{t=1}^n \ind(A_t\in\cS)\Big(\sum_{j=1}^{s-1}|A_{tj}|+\sum_{j\in\supp(\tilde{x})}|A_{tj}|\Big)\leq \sum_{t=1}^n \ind(A_t\in\cS)\sum_{j=1}^d |A_{tj}| \leq (s-1)n,\\
    &\sum_{t=1}^n \ind(A_t\in\cS)\sum_{j=1}^{s-1}|A_{tj}|\leq (s-1)n-\sum_{t=1}^n \ind(A_t\in\cS)\sum_{j\in\supp(\tilde{x})}|A_{tj}|.
\end{split}
\end{equation}
Combining with Eq.~\eqref{eqn:decom2},
\begin{equation*}
\begin{split}
    \sum_{t=1}^n \ind(A_t\in\cS)\langle A_t, \theta \rangle \leq \varepsilon \Big((s-1)n - \sum_{t=1}^n\ind(A_t\in\cS)\sum_{j\in\supp(\tilde{x})}|A_{tj}|\Big)
\end{split}
\end{equation*}
Plugging the above bound into Eq.~\eqref{eqn:decom3}, it holds that 
\begin{equation}\label{eqn:decom6}
     \sum_{t=1}^n\ind(A_t\in\cS)\langle A_t, \tilde{\theta}\rangle \leq \varepsilon (s-1)n + \varepsilon \sum_{t=1}^n\ind(A_t\in\cS)\sum_{j\in\supp(\tilde{x})}|A_{tj}|.
\end{equation}
When the event $\cD^c$ (the complement event of $\cD$) happen, we have
\begin{equation*}
    \sum_{t=1}^n \ind(A_t\in\cS)\sum_{j=1}^{s-1}|A_{tj}|\geq \sum_{t=1}^n \ind(A_t\in\cS)\sum_{j=1}^{s-1}A_{tj}\geq \frac{n(s-1)}{2}.
\end{equation*}
Combining with Eq.~\eqref{eqn:decom4}, we have under event $\cD^c$, 
\begin{equation}\label{eqn:decom7}
    \sum_{t=1}^n\ind(A_t\in\cS)\sum_{j\in\supp(\tilde{x})}|A_{tj}| \leq \frac{n(s-1)}{2}.
\end{equation}
Putting Eqs.~\eqref{eqn:decom5}, \eqref{eqn:decom6}, \eqref{eqn:decom7} together, it holds that
\begin{equation}\label{eqn:bound_R_alt}
     R_{\tilde{\theta}}(n) \geq \frac{n(s-1)\varepsilon}{2}\mathbb P_{\tilde{\theta}}(\cD^c). 
\end{equation}
This ends the proof.

\subsection{Proof of Claim \ref{claim:KL_bound}}\label{sec:claim_KL}
From the divergence decomposition lemma (Lemma \ref{eqn:inf-processing} in the appendix), we have
\begin{equation*}
    \begin{split}
        \KL\big(\mathbb P_{\theta}, \mathbb P_{\tilde{\theta}}\big)&= \frac{1}{2}\mathbb E _{\theta}\Big[\sum_{t=1}^n\langle A_t, \theta-\tilde{\theta}\rangle^2\Big]\\
         &= 2\varepsilon^2\mathbb E _{\theta}\Big[\sum_{t=1}^n\langle A_t, \tilde{x}\rangle^2\Big].
    \end{split}
\end{equation*}
To prove the claim, we use a simple argument ``minimum is always smaller than the average". We decompose the following summation over action set $\cS'$ defined in Eq.~\eqref{def:S_prime},
\begin{equation*}
    \begin{split}
        \sum_{x\in\cS'} \sum_{t=1}^n \langle A_t, x \rangle^2&= \sum_{x\in\cS'} \sum_{t=1}^n \Big(\sum_{j=1}^dx_j A_{tj}\Big)^2\\
      &=  \sum_{x\in\cS'} \sum_{t=1}^n \Big(\sum_{j=1}^d\big(x_j A_{tj}\big)^2 + 2\sum_{i<j}x_ix_jA_{ti}A_{tj}\Big).
    \end{split}
\end{equation*}
We bound the above two terms separately. 
\begin{enumerate}
    \item To bound the first term, we observe that 
    \begin{equation}\label{eqn:decom_1}
    \begin{split}
         &\sum_{x\in\cS'} \sum_{t=1}^n \sum_{j=1}^d\big(x_j A_{tj}\big)^2 \\
         =& \sum_{x\in\cS'} \sum_{t=1}^n\ind(A_t\in \cS)\sum_{j=1}^d|x_j A_{tj}|+ \sum_{x\in\cS'} \sum_{t=1}^n\ind(A_t\in \cH)\sum_{j=1}^d(x_j A_{tj})^2,
    \end{split}
    \end{equation}
    since both $x_j, A_{tj}$ can only take $-1, 0, +1$ if $A_t\in\cS$. If all the $A_t$ come from $\cS$, we have 
    \begin{equation*}
    \sum_{t=1}^n \sum_{j=1}^d |A_{tj}|= (s-1)n.
\end{equation*}
This implies 
\begin{equation*}
    \sum_{t=1}^n \ind(A_t\in\cS)\sum_{j=1}^d |A_{tj}| \leq (s-1)n.
\end{equation*}
Since $x\in\cS'$ that is $(s-1)$-sparse, we have $\sum_{j=1}^d|x_j A_{tj}| \leq s-1$. Therefore, we have
\begin{equation}\label{eqn:bound1}
\begin{split}
    \sum_{x\in\cS'}\sum_{t=1}^n \ind(A_t\in\cS)\sum_{j=1}^d |x_jA_{tj}|\leq (s-1)n \binom{d-s-1}{s-2}.
\end{split}
\end{equation}
In addition, since the action in $\cS'$ is $s-1$-sparse and has 0 at its last coordinate, we have
\begin{equation}\label{eqn:bound11}
     \sum_{x\in\cS'} \sum_{t=1}^n\ind(A_t\in \cH)\sum_{j=1}^d(x_j A_{tj})^2\leq \kappa^2|\cS'|T_n(\cH)(s-1).
\end{equation}
Putting Eqs.~\eqref{eqn:decom_1}, \eqref{eqn:bound1} and \eqref{eqn:bound11} together,
\begin{equation}\label{eqn:sum1}
     \sum_{x\in\cS'} \sum_{t=1}^n \sum_{j=1}^d\big(x_j A_{tj}\big)^2 \leq (s-1)n \binom{d-s-1}{s-2} +  \kappa^2|\cS'|T_n(\cH)(s-1) .
\end{equation}
\item To bound the second term, we observe 
\begin{equation*}
    \sum_{x\in\cS'} \sum_{t=1}^n  2\sum_{i<j}x_ix_jA_{ti}A_{tj}=2\sum_{t=1}^n \sum_{i<j} \sum_{x\in\cS'} x_ix_jA_{ti}A_{tj}.
\end{equation*}
From the definition of $\cS'$, $x_ix_j$ can only take values of $\{1*1, 1*-1, -1*1, -1*-1, 0\}$. This symmetry implies 
\begin{equation*}
    \sum_{x\in\cS'}x_ix_jA_{ti}A_{tj} = 0,
\end{equation*}
which implies 
\begin{equation}\label{eqn:sum2}
     \sum_{x\in\cS'} \sum_{t=1}^n  2\sum_{i<j}x_ix_jA_{ti}A_{tj} = 0.
\end{equation}
\end{enumerate}
Combining Eqs.~\eqref{eqn:sum1} and \eqref{eqn:sum2} together, we have
\begin{equation*}
\begin{split}
\sum_{x\in\cS'} \sum_{t=1}^n \langle A_t, x \rangle^2&= \sum_{x\in\cS'} \sum_{t=1}^n \sum_{j=1}^d|x_j A_{tj}|\\
&\leq (s-1)n \binom{d-s-1}{s-2}+\kappa^2|\cS'|T_n(\cH)(s-1).
\end{split}
\end{equation*}
Therefore, we use the fact that the minimum of $n$ points is always smaller than its average,
\begin{equation*}
\begin{split}
     \mathbb E_{\theta}\Big[\sum_{t=1}^n \langle A_t, \tilde{x} \rangle^2\Big] &=  \min_{x\in\cS'}\mathbb E_{\theta}\Big[\sum_{t=1}^n \langle A_t, x\rangle^2\Big]\\
     &\leq \frac{1}{|\cS'|}\sum_{x\in\cS'}\mathbb E_{\theta}\Big[\sum_{t=1}^n \langle A_t, x \rangle^2\Big]\\
     &=\mathbb E_{\theta}\Big[\frac{1}{|\cS'|}\sum_{x\in\cS'} \sum_{t=1}^n \langle A_t, x \rangle^2\Big]\\
    &\leq \frac{(s-1)n\binom{d-s-1}{s-2}+\mathbb E_{\theta}[T_n(\cH)](s-1)\binom{d-s}{s-1}}{\binom{d-s}{s-1}}\\
    &\leq \frac{(s-1)^2 n}{d}+\kappa^2\mathbb E_{\theta}[T_n(\cH)](s-1).
\end{split}
\end{equation*}
This ends the proof of the claim of Eq.~\eqref{eqn:KL_bound}.

\subsection{Proof of Theorem \ref{thm:upper_bound}: regret upper bound}\label{sec:proof_upper}
\textbf{Step 1: regret decomposition.} Suppose $R_{\max}$ is an upper bound of maximum expected reward such that $ \max_{x\in\cA}\langle x,\theta \rangle\leq R_{\max}$. We decompose the regret of ESTC as follows:
\begin{equation*}
    \begin{split}
           R_{\theta}(n) &= \mathbb E_{\theta}\Big[\sum_{t=1}^n\big\langle \theta, x^*-A_t \big\rangle\Big]\\
    &= \mathbb E_{\theta}\Big[\sum_{t=1}^{n_1}\big\langle \theta, x^*-A_t \big\rangle + \sum_{t = n_1+1}^n\big\langle \theta, x^*-A_t\big\rangle \Big]\\
    &\leq \mathbb E_{\theta}\Big[2n_1R_{\max} +\sum_{t=n_1+1}^n\big\langle \theta-\hat{\theta}_{n_1}, x^*-A_t\big\rangle + \sum_{t=n_1+1}^n\big\langle \hat{\theta}_{n_1}, x^*-A_t\big\rangle \Big].
    \end{split}
\end{equation*}
Since we take greedy actions when $t\geq n_1+1$, it holds that $\langle x^*, \hat{\theta}_{n_1}\rangle\leq \langle A_t, \hat{\theta}_{n_1}\rangle$. This implies
\begin{equation}\label{eqn:regret_decom}
\begin{split}
    R_{\theta}(n)  &\leq \mathbb E_{\theta}\Big[2n_1R_{\max} +\sum_{t=n_1+1}^n\big\langle \theta-\hat{\theta}_{n_1}, x^*-A_t\big\rangle \Big]\\
    &\leq \mathbb E_{\theta}\Big[2n_1R_{\max} + \sum_{t=n_1+1}^n\big\|\theta-\hat{\theta}_{n_1}\big\|_1\big\|x^*-A_t\big\|_{\infty}\Big].
\end{split}
\end{equation}

\textbf{Step 2: fast sparse learning.} It remains to bound the estimation error of $\hat{\theta}_{n_1}-\theta$ in $\ell_1$-norm. Denote the design matrix $X = (A_1, \ldots, A_{n_1})^{\top}\in\mathbb R^{n_1\times d}$, where $A_1, \ldots, A_{n_1}$ are independently drawn according to sampling distribution $\hat{\mu}$. To achieve a fast rate, one need to ensure $X$ satisfies restricted eigenvalue condition (Condition \ref{con:RE} in the appendix). Denote the uncentered empirical covariance matrix  $\hat{\Sigma} = X^{\top}X/n_1$. It is easy to see
\begin{equation*}
\Sigma =  \mathbb E(\hat{\Sigma}) = \int_{x\in \cA}xx^{\top} d\hat{\mu}(x),
\end{equation*}
where $\hat{\mu}$ is the solution of optimization problem Eq.~\eqref{def:opti}.
To lighten the notation, we write $C_{\min} = C_{\min}(\cA)$. Since action set $\cA$ spans $\mathbb R^d$, we know that $\sigma_{\min}(\Sigma) = C_{\min}>0$. And we also denote $\sigma_{\max}(\Sigma) = C_{\max}$ and the notion of restricted eigenvalue as follows.
\begin{definition}
Given a symmetric matrix $H\in\mathbb R^{d\times d}$ and integer $s\geq 1$, and $L>0$, the restricted eigenvalue of $H$ is defined as
\begin{equation*}
    \phi^2(H, s, L):=\min_{\cS\subset [d], |\cS|\leq s}\min_{\theta\in\mathbb R^d}\Big\{\frac{\langle \theta, H\theta \rangle}{\|\theta_{\cS}\|_1^2}: \theta\in\mathbb R^d, \|\theta_{\cS^c}\|_1\leq L\|\theta_{\cS}\|_1\Big\}.
\end{equation*}
\end{definition}

It is easy to see $X\Sigma^{-1/2}$ has independent sub-Gaussian rows with sub-Gaussian norm $\|\Sigma^{-1/2}A_1\|_{\psi_2} = C_{\min}^{-1/2}$ (see \cite{vershynin2010introduction} for a precise definition of sub-Gaussian rows and sub-Gaussian norms). According to Theorem 10 in \cite{javanmard2014confidence} (essentially from Theorem 6 in \cite{rudelson2013reconstruction}), if the population covariance matrix 
satisfies the restricted eigenvalue condition, the empirical covariance matrix satisfies it as well with high
probability. Specifically, suppose the rounds in the exploration phase satisfies  $n_1\geq 4c_*mC_{\min}^{-2} \log (ed/m)$ for some $c_*\leq 2000$ and $m=10^4s C^2_{\max}/\phi^2(\Sigma, s, 9)$. Then the following holds:
\begin{equation*}
    \mathbb P\Big(\phi(\hat{\Sigma}, s, 3)\geq \frac{1}{2}\phi(\Sigma, s, 9)\Big)\geq 1-2\exp(-n_1/(4c_*C_{\min}^{-1/2})).
\end{equation*}
 Noticing that $\phi(\Sigma, s, 9)\geq C_{\min}^{1/2}$, it holds that
\begin{equation*}
     \mathbb P\Big(\phi^2(\hat{\Sigma}, s, 3)\geq \frac{C_{\min}}{2}\Big)\geq 1-2\exp(-c_1n_1),
\end{equation*}
where $c_1 = 1/(4c^*C_{\min}^{-1/2})$.
This guarantees $\hat{\Sigma}$ satisfies Condition \ref{con:RE} in the appendix with $\kappa = C_{\min}/2$. It is easy to see Condition \ref{con:column_normal} holds automatically. Applying Theorem \ref{thm:Lasso} in the appendix of the Lasso error bound, it implies:
\begin{equation*}
    \big\|\hat{\theta}_{n_1} -\theta^*\big\|_1\leq \frac{2}{C_{\min}}\sqrt{\frac{2s^2(\log (2d)+\log(n_1))}{n_1}}.
\end{equation*}
with probability at least $1-\exp(-n_1)$. 

\textbf{Step 3: optimize the length of exploration.} Define an event $\cE$ as follows:
\begin{equation*}
    \cE = \Big\{\phi(\hat{\Sigma}, s, 3)\geq \frac{C_{\min}^{1/2}}{2},   \big\|\hat{\theta}_{n_1} -\theta^*\big\|_1\leq \frac{2}{C_{\min}}\sqrt{\frac{2s^2(\log (2d)+\log(n_1))}{n_1}}\Big\}.
\end{equation*}
We know that $\mathbb P(\cE)\geq 1- 3\exp(-c_1n_1)$. Note that $\|x^*-A_t\|_{\infty}\leq 2$. According to Eq.~\eqref{eqn:regret_decom}, we have
\begin{equation*}
    \begin{split}
        R_{\theta}(n)&\leq \mathbb E_{\theta}\Big[\Big(2n_1R_{\max} + \sum_{t=n_1+1}^n\big\|\theta-\hat{\theta}_{n_1}\big\|_1\big\|x^*-A_t\big\|_{\infty}\Big)\ind(\cE)\Big] + nR_{\max}\mathbb P(\cE^c)\\
        &\leq n_1R_{\max} + (n-n_1)\frac{4}{C_{\min}}\sqrt{\frac{2s^2 (\log (2d)+\log(n_1))}{n_1}}2 + 3nR_{\max}\exp(-c_1n_1)
    \end{split}
\end{equation*}
with probability at least $1-\delta$. By choosing $n_1 = n^{2/3}(s^2\log (2d))^{1/3}R_{\max}^{-2/3} (2/C_{\min}^2)^{1/3}$, we have
\begin{equation*}
    R_{\theta}(n)\leq (sn)^{2/3}(\log (2d))^{1/3}R_{\max}^{1/3}(\frac{2}{C_{\min}^2})^{1/3} + 3nR_{\max}\exp(-c_1n_1).
\end{equation*}
We end the proof.

\subsection{Proof of Theorem \ref{thm:improved_upper_bound}: improved regret upper bound}\label{sec:proof_improved_bound}

We start from a simple regret decomposition based on feature selection step and restricted linear bandits step: 
\begin{equation*}
    \begin{split}
           R_{\theta}(n) &= \mathbb E_{\theta}\Big[\sum_{t=1}^n\big\langle \theta, x^*-A_t \big\rangle\Big]\\
    &= \mathbb E_{\theta}\Big[2n_2 R_{\max} + \sum_{t = n_2+1}^n\big\langle \theta, x^*-A_t\big\rangle \Big].
    \end{split}
\end{equation*}
\textbf{Step 1: sparsity property of Lasso.} We first prove that the Lasso solution is sufficiently sparse. The following proof is mainly from \cite{bickel2009simultaneous} with minor changes. To be self-contained, we reproduce it here.   Recall that the Lasso estimator in the feature selection stage is defined as 
\begin{equation*}
      \hat{\theta} = \argmin_{\theta\in\mathbb R^d}\Big(\frac{1}{n_2}\sum_{t=1}^{n_2}\big(Y_t-\langle A_t, \theta\rangle\big)^2 + \lambda_2\|\theta\|_1\Big).
\end{equation*}
Define random variables $V_j = \frac{1}{n_2}\sum_{t=1}^{n_2}A_{tj}\eta_t$ for $j\in[d]$ and $\eta_t$ is the noise. Since $\|A_{t}\|_{\infty}\leq 1$, standard Hoeffding's inequality (Proposition 5.10 in \cite{vershynin2010introduction}) implies
\begin{equation*}
    \mathbb P\Big(\big|\sum_{t=1}^{n_2}A_{tj}\eta_t\big|\geq \varepsilon\Big) \leq \exp\Big(-\frac{\varepsilon^2}{2n_2}\Big).
\end{equation*}
Define an event $\cE$ as 
\begin{equation*}
    \cE = \bigcup_{j=1}^d \Big\{|V_j|\leq \sqrt{\frac{4\log (d)}{n_2}}\Big\}.
\end{equation*}
Using an union bound, we have 
\begin{equation*}
    \mathbb P(\cE^c) \leq 1/d.
\end{equation*}

From the Karush–Kuhn–Tucker (KKT) condition, the solution $\hat{\theta}$ satisfies
\begin{equation}\label{eqn:KKT}
    \begin{split}
       & \frac{1}{n_2}\sum_{t=1}^{n_2} A_{tj}^{\top}(Y_t - A_t^{\top}\hat{\theta})= \lambda_2 \text{sign}(\hat{\theta}_{j}), \ \text{if} \ \hat{\theta}_j\neq 0;\\
       &\Big|\frac{1}{n_2}\sum_{t=1}^{n_2}A_{tj}^{\top}(Y_t - A_t^{\top}\hat{\theta})\Big|\leq \lambda_2, \ \text{if} \ \hat{\theta}_j = 0.
    \end{split}
\end{equation}
Therefore,
\begin{equation*}
    \frac{1}{n_2}\sum_{t=1}^{n_2}A_{tj}(A_t^{\top}\theta - A_t^{\top}\hat{\theta}) = \frac{1}{n_2}\sum_{i=1}^{n_2}A_{tj}(Y_t - A_t^{\top}\hat{\theta})-\frac{1}{n_2}\sum_{i=1}^{n_2}A_{tj}\eta_t
\end{equation*}
Since $\lambda_2= 4\sqrt{\log (d)/n_2}$, under event $\cE$, we have 
\begin{equation*}
    \Big| \frac{1}{n_2}\sum_{t=1}^{n_2}A_{tj}(A_t^{\top}\theta - A_t^{\top}\hat{\theta})\Big| \geq \lambda_2/2, \ \text{if} \ \hat{\theta}_j\neq 0.
\end{equation*}
And
\begin{equation*}
    \begin{split}
        \frac{1}{n_2^2}\sum_{j=1}^d \Big(\sum_{t=1}^{n_2}A_{tj}(A_t^{\top}\theta - A_t^{\top}\hat{\theta})\Big)^2&\geq  \sum_{j:\hat{\theta}_j\neq 0} \Big(\frac{1}{n_2}\sum_{t=1}^{n_2}A_{tj}(A_t^{\top}\theta - A_t^{\top}\hat{\theta})\Big)^2\\
        &\geq |\supp(\hat{\theta}_{n_2})|\lambda_2^2/4.
   \end{split}
\end{equation*}

On the other hand, let $X = (A_1, \ldots, A_{n_2})^{\top}\in\mathbb R^{n_2\times d}$ and $\phi_{\max} = \sigma_{\max}(XX^{\top}/n_2)$. Then we have 
\begin{equation*}
    \begin{split}
         &\frac{1}{n_2^2}\sum_{j=1}^d \Big(\sum_{t=1}^{n_2}A_{tj}\Big(A_t^{\top}\theta - A_t^{\top}\hat{\theta}\Big)\Big)^2\\
         =&\frac{1}{n_2^2} \Big(X\theta - X\hat{\theta}\Big)^{\top}XX^{\top}\Big(X\theta - X\hat{\theta}\Big)\leq \phi_{\max}\frac{1}{n_2}\|X\hat{\theta}-X\theta\|_2^2.
    \end{split}
\end{equation*}
Therefore, with probability at least $1-1/d$,
\begin{equation}\label{eqn:123}
    |\supp(\hat{\theta}_{n_2})| \leq \frac{4\phi_{\max}}{\lambda_2^2 n_2}\|X\hat{\theta}-X\theta\|_2^2.
\end{equation}
To lighten the notation, we write $C_{\min} = C_{\min}(\cA)$. As proven in Section \ref{sec:proof_upper}, $X^{\top}X/n_2$ satisfies Condition \ref{con:RE}  with $\kappa = C_{\min}/2$ when $n_2\gtrsim s\log (d)$. Applying the in-sample prediction error bound in Theorem \ref{thm:Lasso}, we have with probability at least $1-1/d$,
\begin{equation}\label{eqn:2}
   \frac{1}{n_2} \big\|X\hat{\theta} - X\theta\big\|_2^2 \leq \frac{9}{C_{\min}}\frac{s\log (d)}{n_2}.
\end{equation}
Putting Eqs.~\eqref{eqn:123} and \eqref{eqn:2} together, we have with probability at least $1-2/d$.
\begin{equation}\label{eqn:support}
    |\supp(\hat{\theta})| \leq \frac{9\phi_{\max}s}{C_{\min}}.
\end{equation}

\textbf{Step 2: variable screening property of Lasso.} Under Condition \ref{con:min_signal} and using Theorem \ref{thm:Lasso}, it holds that with probability at least $1-1/d$,
\begin{equation*}
   \min_{j\in \supp(\theta)} |\theta_j| > \big\|\hat{\theta} -\theta\big\|_2\geq \big\|\hat{\theta} -\theta\big\|_{\infty}.
\end{equation*}
If there is a $j\in \supp(\theta)$ but $j\notin\supp(\hat{\theta})$, we have 
\begin{equation*}
    |\hat{\theta}_j-\theta_j| = |\theta_j|> \big\|\hat{\theta} -\theta\big\|_{\infty}.
\end{equation*}
On the other hand,
\begin{equation*}
    |\hat{\theta}_j-\theta_j|\leq \big\|\hat{\theta} -\theta\big\|_{\infty},
\end{equation*}
which leads a contradiction. Now we conclude that $\supp(\hat{\theta})\supseteq \supp(\theta)$. We reproduce Theorem 22.1 in \cite{lattimore2018bandit} for the regret bound of phase elimination algorithm for stochastic linear bandits with finitely-many arms.
\begin{theorem}
The $n$-steps regret of phase elimination algorithm satisfies
\begin{equation*}
    R_n \leq C\sqrt{nd \log(Kn)},
\end{equation*}
for an appropriately chosen universal constant $C>0$.
\end{theorem}
Together with Eq.~\eqref{eqn:support}, we argue the regret of running phase elimination algorithm  (Section 22 in \cite{lattimore2018bandit}) on $\supp(\hat{\theta})$ for the rest $n-n_2$ rounds can be upper bounded by 
\begin{equation*}
    \mathbb E_{\theta}\Big[\sum_{t = n_2+1}^n\big\langle \theta, x^*-A_t\big\rangle \Big]\leq C\sqrt{\frac{9\phi_{\max}}{C_{\min}}s(n-n_2)\log (K(n-n_2))}.
\end{equation*}
This ends the proof.

\section{Supporting lemmas}\label{sec:supporting_lemmas}

\begin{lemma}[Bretagnolle-Huber inequality]\label{lem:kl}
Let $\mathbb P$ and $\tilde{\mathbb P}$ be two probability measures on the same measurable space $(\Omega,\cF)$. Then for any event $\cD\in \cF$,
\begin{equation}\label{eqn:kl}
\mathbb P(\cD) + \tilde{\mathbb P}(\cD^c) \geq \frac{1}{2} \exp\left(-\text{KL}(\mathbb P, \tilde{\mathbb P})\right)\,,
\end{equation}
where $\cD^c$ is the complement event of $\cD$ ($\cD^c = \Omega\setminus \cD$) and $\text{KL}(\mathbb P, \tilde{\mathbb P})$ is the KL divergence between $\PP$ and $\tilde{\mathbb P}$, which is defined as $+\infty$, if $\PP$ is not absolutely continuous with respect to $\tilde{\mathbb P}$, and is $\int_\Omega d\PP(\omega) \log \frac{d\PP}{d\tilde{\mathbb P}}(\omega)$ otherwise.
\end{lemma}
The proof can be found in the book of \cite{Tsybakov:2008:INE:1522486}. When $\text{KL}(\mathbb P, \tilde{\mathbb P})$ is small, we may expect the probability measure $\mathbb P$ is close to the probability measure $\tilde{\mathbb P}$. Note that $\mathbb P(\cD) + \mathbb P(\cD^c)=1$. If $\tilde{\mathbb P}$ is close to $\mathbb P$, we may expect $\mathbb P(\cD)+\tilde{\mathbb P}(\cD^c)$ to be large.

\begin{lemma}[Divergence decomposition]\label{lem:inf-processing}
Let $\mathbb P$ and $\tilde{\mathbb P}$ 
be two probability measures on the sequence  $(A_1, Y_1,\ldots,A_n,Y_n)$ for a fixed
bandit policy $\pi$ interacting with a linear contextual bandit with standard Gaussian noise and parameters $\theta$ and $\tilde{\theta}$ respectively. Then the KL divergence of $\mathbb P$ and $\tilde{\mathbb P}$ can be computed  exactly and is given by
\begin{equation}\label{eqn:inf-processing}
\text{KL}(\mathbb P, \tilde{\mathbb P}) = \frac12 \sum_{x \in \mathcal A} \mathbb E[T_x(n)]\, \langle x, \theta - \tilde{\theta}\rangle^2\,,
\end{equation}
where $\mathbb E$ is the expectation operator induced by $\PP$. 
\end{lemma}
This lemma appeared as Lemma 15.1 in the book of \cite{lattimore2018bandit}, where the reader can also find the proof.

\end{document}